%% file: GORAG_WWW.tex
\documentclass[sigconf]{acmart}
\usepackage{multirow}
\usepackage{algorithm}
\usepackage{algorithmic}
\usepackage{enumerate}
\usepackage{subfigure}
\usepackage{makecell}
\usepackage{enumitem}
\usepackage{balance}
\usepackage{xcolor}

\definecolor{mygreen}{RGB}{126,172,85}
\definecolor{myblue}{RGB}{78,113,190}
\definecolor{myorange}{RGB}{184,96,41}
\AtBeginDocument{%
  }

\copyrightyear{2026}
\acmYear{2026}
\setcopyright{cc}
\setcctype{by}
\acmConference[WWW '26]{Proceedings of the ACM Web Conference 2026}{April 13--17, 2026}{Dubai, United Arab Emirates}
\acmBooktitle{Proceedings of the ACM Web Conference 2026 (WWW '26), April 13--17, 2026, Dubai, United Arab Emirates}
\acmDOI{10.1145/3774904.3793011}
\acmISBN{979-8-4007-2307-0/2026/04}

\begin{document}

\title{GORAG: Graph-based Online Retrieval Augmented Generation for Dynamic Few-shot Social Media Text Classification}
\settopmatter{authorsperrow=4,printacmref=true}
\author{Yubo Wang}
\affiliation{%
	\institution{The Hong Kong University of Science and Technology}
	\city{Hong Kong SAR}
	\country{China}
}
\email{ywangnx@connect.ust.hk}

\author{Haoyang Li}
\authornote{Corresponding author}
\affiliation{%
	\institution{The Hong Kong Polytechnic University}
	\city{Hong Kong SAR}
	\country{China}
}
\email{haoyang-comp.li@polyu.edu.hk}

\author{Fei Teng}
\affiliation{%
	\institution{The Hong Kong University of Science and Technology}
	\city{Hong Kong SAR}
	\country{China}
}
\email{fteng@connect.ust.hk}

\author{Lei Chen}
\affiliation{%
	\institution{The Hong Kong University of Science and Technology (Guangzhou)}
	\city{Guangzhou}
	\country{China}
}
\affiliation{%
	\institution{Guangzhou HKUST Fok Ying Tung Research Institute}
	\city{Guangzhou}
	\country{China}
}
\email{leichen@hkust-gz.edu.cn}

\renewcommand{\shortauthors}{Yubo Wang, Haoyang Li, Fei Teng, and Lei Chen}

\begin{abstract}
Text classification is vital for Web for Good applications like hate speech and misinformation detection. However, traditional models (e.g., BERT) often fail in dynamic few-shot settings where labeled data are scarce, and target labels frequently evolve. While Large Language Models (LLMs) show promise in few-shot settings, their performance is often hindered by increased input size in dynamic evolving scenarios. To address these issues, we propose \textbf{GORAG}, a \textbf{G}raph-based \textbf{O}nline \textbf{R}etrieval-\textbf{A}ugmented \textbf{G}eneration framework for dynamic few-shot text classification. GORAG constructs and maintains a weighted graph of keywords and text labels, representing their correlations as edges. To model these correlations, GORAG employs an edge weighting mechanism to prioritize the importance and reliability of extracted information and dynamically retrieves relevant context using a tailored minimum-cost spanning tree for each input. Empirical evaluations show GORAG outperforms existing approaches by providing more comprehensive and precise contextual information. 
Our code is released at: \url{https://github.com/Wyb0627/GORAG}.
\end{abstract}




\maketitle
\vfill\eject
\input{section/sec-introduction.tex}
\input{section/sec-related_work.tex}
\input{section/sec-method.tex}

\input{section/sec-experiment_results.tex}

\input{section/sec-conclusion.tex}
\section{Acknowledgements}
Lei Chen’s work is partially supported by National Key Research and Development Program of China Grant No. 2023YFF0725100, National Science Foundation of China (NSFC) under Grant No. U22B2060, Guangdong-Hong Kong Technology Innovation Joint Funding Scheme Project No. 2024A0505040012, AOE Project AoE/E-603/18, Theme-based project TRS T41-603/20R, CRF Project C2004-21G, Key Areas Special Project of Guangdong Provincial Universities 2024ZDZX1006,  Guangdong Province Science and Technology Plan Project 2023A0505030011, Guangzhou municipality big data intelligence key lab, 2023A03J0012, HKUST(GZ) - CMCC(Guangzhou Branch) Metaverse Joint Innovation Lab under Grant No. P00659, Hong Kong ITC ITF grant PRP/004/22FX, Zhujiang scholar program 2021JC02X170, HKUST-Webank joint research lab.
\clearpage

\bibliographystyle{ACM-Reference-Format}
\balance
\bibliography{GraphRAG_FSCIL}
\clearpage
\appendix
\input{section/sec-appendix.tex}

\end{document}

%% file: section/sec-introduction.tex
\section{Introduction}
\begin{figure}[t]
	\centering
	\includegraphics[width=1\linewidth]{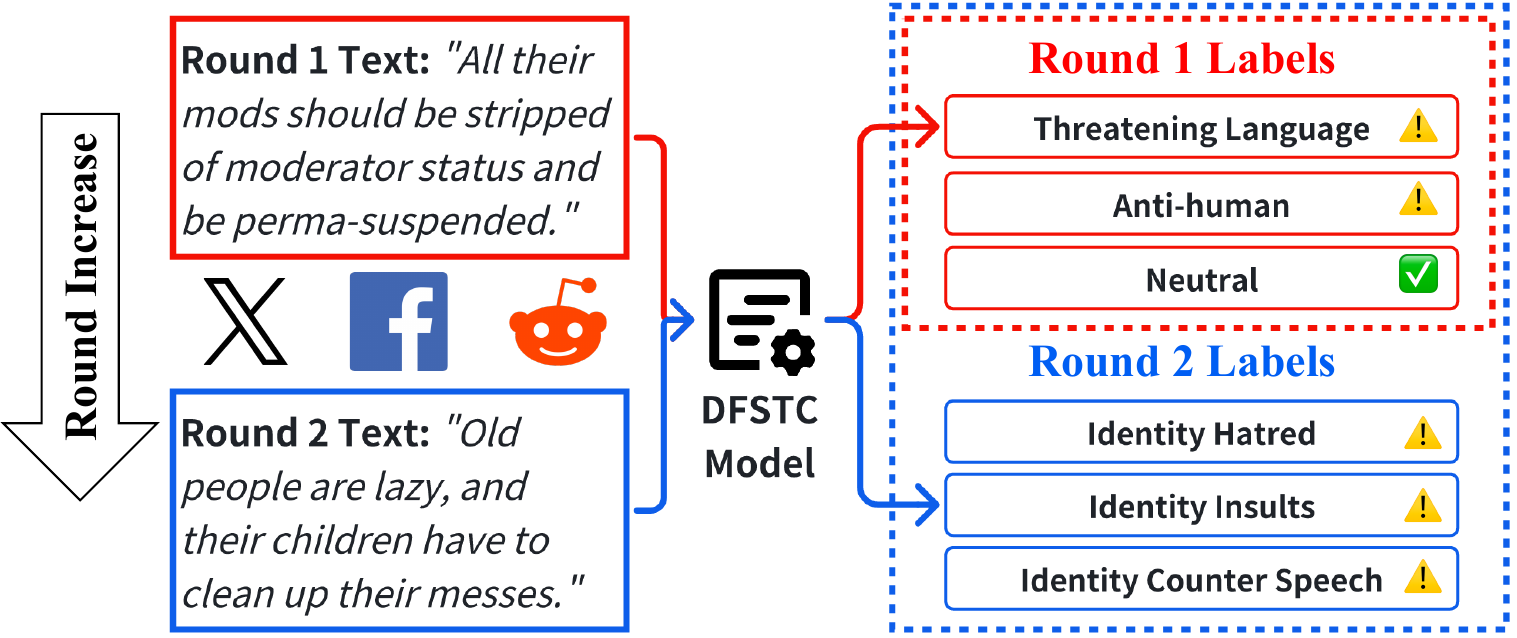}
	\vspace{-18px}
	\caption{An example of DFSTC task with 2 rounds. 
		In Round 1, the text is classified as \textit{Threatening Language}.
		In Round 2, another text is classified as a new label: \textit{Identity Insults}.
	}
	\label{fig:dfstc_task}
	\vspace{-15px}
\end{figure}

Text content classification is a fundamental task for various social media applications, 
such as hate speech detection
\cite{geissler2025analyzing,hee2024brinjal,chen2024hate,djuric2015hate,subramanyam2024triage,mandal2024contrastive,cao2024modularized,ji2024capalign,guo2024bayesian},
misinformation/disinformation detection 
\cite{luo2024message,nesterov2024contentious,liu2025modality,toughrai2025fake,cinus2025exposing,ashkinaze2024dynamics,chen2022cross,shang2024mmadapt,ye2024susceptibility}
, and suicide risk detection
\cite{kang2024scenedapr,li2025dynaprotect,naseem2025digri,Ge_2025}. 
In recent years, Neural Network (NN)-based models~\cite{johnson2015semi,liu2016recurrent, dieng2017topicrnn,wang2017bilateral,johnson2017deep,peters-etal-2018-deep,wang2018sentiment,yang2019xlnet}, such BERT~\cite{devlin2018bert} and RoBERTa \cite{liu2019roberta},  have demonstrated impressive performance on social media text content classification tasks. 
However, the effectiveness of these NN-based approaches \cite{johnson2015semi,liu2016recurrent, dieng2017topicrnn,wang2017bilateral,johnson2017deep,peters-etal-2018-deep,wang2018sentiment,yang2019xlnet} depends on abundant labeled training data, which requires significant time and human efforts~\cite{meng2018weakly,nguyen2024human}. 
{Furthermore, in real-world applications, such as 
Abuse Detection \cite{vidgen2021introducing}, new kinds of social media abuse can frequently appear with the changing needs of user analytics and progress in social research development, 
introducing the need to classify texts into new, extra labels
\cite{xia-etal-2021-incremental}, leading to the \textbf{D}ynamic \textbf{F}ew-shot \textbf{S}ocial media \textbf{T}ext \textbf{C}lassification (DFSTC) task. }
As shown in Figure~\ref{fig:dfstc_task}, the DFSTC task starts with an initial set of classes (e.g., \textit{Threatening Language}). As new classes like \textit{Identity Insults} emerge in later rounds, it becomes challenging to obtain extensive labeled data for these new categories.
 Models must therefore adapt to such changes and accurately classify new examples with limited labeled data. 
Developing methods to dynamically classify text with limited labeled data remains a crucial and open problem.

Depending on the technique for 
the DFSTC task,
current models can be mainly categorized into two types, i.e.,
data augmentation-based models and RAG-based models.
Firstly,
data augmentation-based models~\cite{meng2018weakly,meng2020text,xia-etal-2021-incremental} create additional data for fine-tuning small language models (e.g., RoBERTa \cite{liu2019roberta}), by mixing the pairs of the few-shot labeled data and assigning mixed labels indicating the validity for these created data based on the labels of each data pairs~\cite{xia-etal-2021-incremental}.
However, due to the limited labeled data, 
and the label names may not always be available;
furthermore, the additional data created can have very limited patterns. 
Consequently, the classification models trained with these created data may over-fit and are not generalizable~\cite{liu2023pre,liu2024liberating,lei2023tart}.

Recently, Large Language Models (LLMs) \cite{sun2021ernie, jiang2023mistral, touvron2023llama, achiam2023gpt, yang2024qwen2} have excelled in social media text classification through their superior understanding abilities. However, they often struggle with task-specific knowledge \cite{gao2023retrieval,zhang2023siren} and lack accurate reasoning clues to explain responses \cite{peng2024graph,wang2024kglink,dong2025understand,qian2025memorag}. In domains like hate speech or misinformation detection, these flaws can lead to undetected harmful content or wrongful penalties. Consequently, Long Context RAG models \cite{chen2023dense, sarthi2024raptor,xia-etal-2021-incremental,yoon2024compact} were proposed to integrate retrieved side information (e.g., label descriptions) as external clues for more accurate classification.
However, the incorporation of side information can further increase the input size \cite{chen2023dense}, which often overwhelms critical details with noises. As the key clues of hate speech or misinformation are usually short sentences within long contexts, the ignorance of fine-grained, entity-level information can degrade LLM classification performance. Although proprietary LLMs \cite{hurst2024gpt} can handle such extended inputs, they inevitably suffer from high token costs.

To achieve finer-grained RAG, and better preserve the entity-level information, Graph-based RAG approaches~\cite{edge2024local,guo2024lightrag,gutierrez2024hipporag,peng2024graph,zhang2025survey,han2024retrieval,gutiérrez2025ragmemory}, such as GraphRAG~\cite{edge2024local} and LightRAG~\cite{guo2024lightrag}, 
	propose to index the unstructured texts into the structured graphs. 
	These approaches extract entities and relations from texts as graph nodes and edges, and their retrievals are based on the graph components mapped from the query text.
	 However, existing Graph-based RAG approaches~\cite{edge2024local,guo2024lightrag,gutierrez2024hipporag} still have three issues in the DFSTC task.

\begin{enumerate}[leftmargin=10pt]
	\item 
	\textbf{Uniform-indexing
		issue:}
	Some of these approaches construct graphs by indexing extracted text chunks uniformly.
	They do not consider the varying importance of each text and extraction confidence of keyword entity.
		For example, in \autoref{fig:dfstc_task}, keyword \textit{lazy} is related to all round 2 labels with different importance. Hence, 
	uniform-indexing may provide inaccurate graph context for LLM classification.

	\item \textbf{Non-adaptive Retrieval issue:} These approaches select relevant information for each input text based on a globally predefined threshold. 
	However, the optimal retrieval threshold can vary across different query text and different rounds; hence, this non-adaptive retrieval approach can be suboptimal.

	\item	\textbf{Narrow-source issue:} These approaches only retrieve side information from a limited number of few-shot training text and ignore the important information among query texts. 
	Since the new keywords related to certain labels in hate speech or misinformation detection emerge frequently in social media, this issue consequently limits models' retrieval comprehensiveness.

\end{enumerate}

To address these issues, we propose a novel 
\textbf{G}raph-based \textbf{O}nline \textbf{R}etrieval \textbf{A}ugmented \textbf{G}eneration framework for dynamic few-shot social media text classification, called GORAG.
In general, GORAG constructs and maintains an adaptive online weighted graph by extracting side information from all target texts, and tailoring the graph retrieval for each input.

Firstly, GORAG constructs a weighted graph with keywords extracted from labeled texts as keyword nodes, and text labels as label nodes, to help the model better capture graph topology \cite{yang2024effective,zhao2024pre,lin2024psmc}. To model relationships, GORAG employs an edge-weighting mechanism to assign different weights to edges, representing keywords' importance and relevance to the respective text's label. Secondly, GORAG constructs a minimum-cost spanning tree based on query keywords and retrieves the label nodes within it as candidate labels. Since this MST is determined solely by the constructed graph and the keywords of the query text, GORAG achieves adaptive retrieval without human-defined thresholds, ensuring precision for each input. Finally, GORAG applies an online indexing mechanism to enrich the graph with keywords from query texts. By calculating weights that reflect the semantic relevance of keywords to labels, GORAG effectively reduces the potential noise introduced by inaccurate LLM classification during real-time updates, ensuring the graph remains a robust and comprehensive retrieval source

We summarize our novel contributions as follows.
\begin{itemize}[leftmargin=10pt]
\item 
We present an
online RAG framework for DFSTC tasks, namely GORAG, which consists of three steps: Graph Construction, Candidate Label Retrieval, and Text Classification, aiming to address the Uniform-indexing, Non-adaptive Retrieval, and Narrow-source issues.


\item We develop a novel graph edge weighting mechanism based on the text keywords' importance within the text corpus, 
which enables our approach to effectively model the relevance between keywords and labels, 
thereby helping the more accurate retrieval.

\item To achieve adaptive retrieval, 
we formulate the candidate label retrieval problem which is akin to the NP-hard Steiner Tree problem, 
and 
provide an efficient and effective solution based on the greedy algorithm.

\item  Extensive experiments from the social media contextual abuse benchmark and the web news misinformation detection benchmark demonstrate the effectiveness and efficiency of GORAG.



\end{itemize} 
 

%% file: section/sec-related_work.tex
 \section{Preliminary and Related Works}\label{sec:pre_rel}
	We first introduce the preliminaries of Dynamic Few-shot Social media Text Classification (DFSTC) in Section~\ref{sec:dfstc} and then discuss the
		related works for DFSTC task
 in Section~\ref{sec:related_work}. 
	The important notations used in this paper are listed in \autoref{tab:notation} in \autoref{sec:notations}.



\subsection{DFSTC Task} \label{sec:dfstc}
Text classification~\cite{kowsari2019text, minaee2021deep, li2022survey} is a key task in real-world web for good applications, including social media hate speech detection \cite{geissler2025analyzing,hee2024brinjal,chen2024hate,djuric2015hate,subramanyam2024triage,mandal2024contrastive,cao2024modularized,ji2024capalign,guo2024bayesian},
misinformation/disinformation detection 
\cite{luo2024message,nesterov2024contentious,liu2025modality,toughrai2025fake,cinus2025exposing,ashkinaze2024dynamics,chen2022cross,shang2024mmadapt,ye2024susceptibility}.
It involves assigning predefined labels $y \in \mathcal{Y}$
to text $t=(w_1,\cdots,w_{|t|})$ based on its words $w_i$. 
Traditional approaches require fine-tuning with large labeled datasets~\cite{devlin2018bert,liu2019roberta}, which may not always be available.
Recently, few-shot learning addresses this limitation by enabling models to classify social media texts into various classes with only a small number of labeled examples per class~\cite{meng2018weakly,meng2020text}.
On the other hand, dynamic classification~\cite{parisi2019continual, hanmo2024effective, wang2024comprehensive} introduces an additional challenge where new classes are introduced over multiple rounds, requiring the model to adapt to new classes while retaining knowledge of previously seen ones, as frequently involve extensive human labeling for these new classes can be difficult. 
	For example, in \autoref{fig:dfstc_task}, new and more fine-grained classes of hate speech, such as \textit{Identity Insults}, can serve as complements to categories that already exists. 

Combining these aspects, Dynamic Few-Shot Social media Text Classification (DFSTC) task allows the model to handle evolving classification tasks with minimal labeled data, and is better suited for real world scenarios~\cite{xia-etal-2021-incremental}.
In DFSTC task, the model is provided with multiple rounds of new class updates.
Specifically, in each round $r$, a set of new labels $\mathcal{Y}^r_{new}$ is introduced. The complete candidate label set for round $r$, denoted as $\mathcal{Y}^r$, is the union of all new labels from rounds $i = 1$ to $r$: $\mathcal{Y}^r = \bigcup_{i=1}^r \mathcal{Y}^i_{new}$.
 In round $r$, the labeled training set for class labels in $\mathcal{Y}^r$ 
 is denoted as $\mathcal{D}^r=\cup_{y_i \in \mathcal{Y}^r}\{t_j,y_i\}_{j=1}^k$, 
 where per class $y_i \in \mathcal{Y}^r$ only has $k$ labeled examples $\{t_j,y_i\}_{j=1}^k$.
Formally, at the $r$-th round, given the labels in $\mathcal{Y}^r$ and all labeled data $\mathcal{D}^r=\cup_{i=1}^r\mathcal{D}^i$,
the target of DFSTC task is to learn a function $f^r_\theta$, which can learn scores for all target labels $\mathbf{s}^r={f}^r_\theta(t_u,\mathcal{Y}^r)\in[0,1]^{|\mathcal{Y}^r|}$ for the unseen text $t_u$.
Then, we can get the predicted label $y^{r*}\in \mathcal{Y}^r$ for the unseen text $t_u$ as follows.  
\begin{equation}
	y^{r*}=\arg\max_{y^r \in \cup_{i=1}^r \mathcal{Y}^i}({f}_\theta^r(t_u, \cup_{i=1}^r \mathcal{Y}^i)[y^r])
\end{equation}

DFSTC is valuable for its ability to learn from limited labeled data, adapt to evolving class distributions, 
and address real-world scenarios where classes and data evolve over time.

\subsection{DFSTC Models}\label{sec:related_work}
Current Dynamic Few-Shot Social media Text Classification (DFSTC) models can be broadly categorized into 
NN-based models and RAG-based models, RAG-based models can be further categorized into Long Context RAG, and Graph-based RAG models. 

\subsubsection{\textbf{NN-based models}}
NN-based models \cite{meng2018weakly,meng2020text,xia-etal-2021-incremental} generate additional data contrastively based on the few-shot labeled data and the text formed label names to train the NN-based classifier. 
However, due to the limited labeled data, the generated text data of these models can have very limited patterns, which makes them prone to overfitting \cite{liu2023pre,liu2024liberating,lei2023tart}. 
In web for good applications such as hate speech detection, different kinds of hate speech contents are usually a minority compared with neutral contents, this highly unbalanced data distribution can further overfit the NN-based models and affect their performance \cite{liu2023pre}.

\begin{figure*}[ht]
	\centering
	\includegraphics[width=1\linewidth]{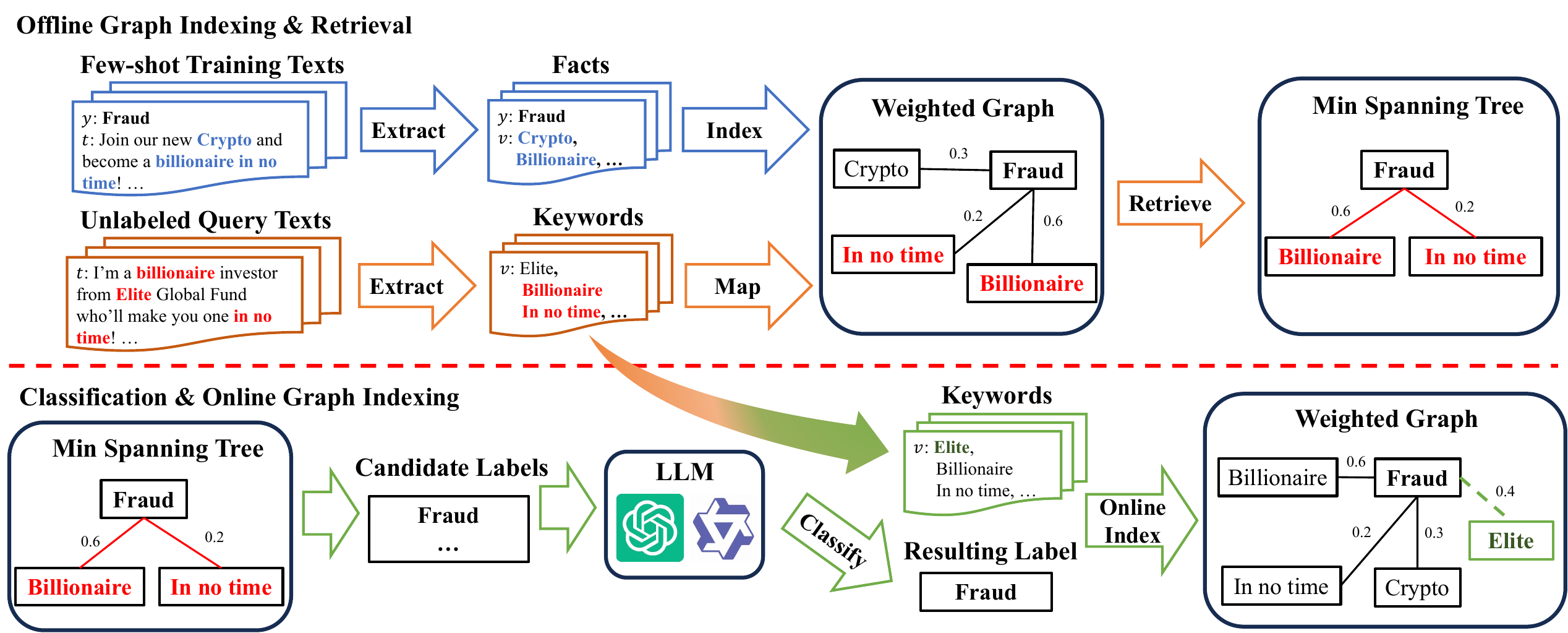}
	\vspace{-20px}
	\caption{An overview of GORAG in each round. 
	In {\color{myblue}Part 1}, GORAG constructs a weighted graph based on keywords extracted from the few-shot training data.  
	In {\color{myorange}Part 2}, GORAG performs adaptive graph retrieval and outputs the candidate labels, which is a subset of the original target label set. 
	In {\color{mygreen}Part 3}, GORAG first classifies texts into one of the candidate labels, 
	and then applies online graph indexing to update the graph with query text extracted entities that not yet existed in the graph.}
	\label{fig:Gorag}
		\vspace{-10px}
\end{figure*}
\subsubsection{\textbf{Long Context RAG models}}
Recently, Large Language Model (LLM)-based models have undergone rapid development 
\cite{zhangJLWC24,snyder2024early,zhong2024logparser,yan2024efficient,chen2024large,ding2024enhancing}.
Notably, LLMs are inherently capable of inference without fine-tuning \cite{sun2021ernie, jiang2023mistral, touvron2023llama, achiam2023gpt, yang2024qwen2}, making them originally suitable for few-shot tasks. However, the lack of fine-tuning can lead LLMs to generate unfaithful responses, as they lack task-specific knowledge, these incorrect answers are often referred to as hallucinations \cite{zhang2023siren}. 
In web for good applications such as hate speech detection, hallucinations can lead to failure to detect harmful responses, making these harmful contents being spread through internet. 
To mitigate hallucinations, researchers 
try to retrieve text documents outside of LLM as side information to help the LLM inference, namely Long Context RAG models
 \cite{chen2023dense,sarthi2024raptor,hu2024prompt,wu2024coral,cai2024forag,oosterhuis2024reliable,yadav2024extreme,che2024hierarchical,pan-etal-2024-llmlingua,jiang-etal-2024-longllmlingua}. 
Among them, prompt compressor models \cite{pan-etal-2024-llmlingua,jiang-etal-2024-longllmlingua}, 
apply LLM's generation perplexity on to filter out un-important tokens in the model input. 

\subsubsection{\textbf{Graph-based RAG models}}
However, the retrieved contents of Long Context RAG models remain unstructured and lengthy, making the important fine-grained entity-level information be overwhelmed by noisy contexts \cite{edge2024local,guo2024lightrag}, which still impedes the effectiveness or efficiency of LLM classification, 
leading to lost-in-the-middle issues \cite{liu2024lost} for cheaper open-sourced LLMs, or high token cost for proprietary LLMs \cite{hurst2024gpt}. 
To address these issues, Graph-based RAG models were proposed
\cite{peng2024graph,zhang2024graph,wu2024medical,zhang2024knowgpt,guo2024lightrag,gutierrez2024hipporag,gutiérrez2025ragmemory},
 these models
convert text chunks from the side information document as graph nodes, and the relation between these text chunks as graph edges. These models then retrieve graph nodes and edges as the LLM classification context, better preserving more fine-grained, entity-level information.
However, current Graph-based RAG models usually consider each graph edge uniformly and ignore the varying confidence and importance of relations between text chunks. Also, as the optimal retrieval threshold can vary across different data samples, their globally fixed threshold selected by humans may be suboptimal for the entire dataset. 

%% file: section/sec-method.tex
\section{Methodology}\label{sec:methdology}
\subsection{Framework Overview}

	To address the aforementioned Uniform-indexing, Non-adaptive Retrieval, and Narrow-source issues,
we propose GORAG, a novel approach that achieves adaptive retrieval by extracting valuable side information from a minimum-cost spanning tree generated on the constructed graph. 
As shown in \autoref{fig:Gorag}, GORAG consists of three core parts, i.e., \textcolor{myblue}{\textbf{Offline Graph Indexing}}, \textcolor{myorange}{\textbf{Graph Retrieval}}, and \textcolor{mygreen}{\textbf{Classification \& Online Graph Indexing}}.

\noindent\textbf{Part 1: Offline Graph Indexing. } 
It targets to construct the weighted graph $\mathcal{G}^r_{new}(\mathcal{V}^r_{new}, \mathcal{E}^r_{new},\mathcal{W}^r_{new})$ based on the labeled data $\mathcal{D}^r=\cup_{y_i \in \mathcal{Y}^r}\{t_j,y_i\}_{j=1}^k$ 
	at the $r$-th round. 
	An example weighted graph created is shown in \autoref{fig:Gorag}, this graph will be used to provide retrieved-augmented information as contexts for query texts, 
	enabling LLMs to better understand query texts and accurately classify them.
	Specifically, with the new labeled data $\mathcal{D}^r$ at the $r$-th round, we first extract keywords $\mathcal{K}^r \in \{t_j\}_{j=1}^{|\mathcal{D}^r|}$.
	Then, we will assign an edge $e^r_{v,y} \in \mathcal{E}^r_{new}$ between each keyword  $v \in \mathcal{K}^r$ and their respective text's label $y\in \mathcal{Y}^r$, {\color{black}where $\mathcal{Y}^r = \bigcup_{i=1}^r \mathcal{Y}^i_{new}$ denotes the union label set 
		of all new labels $\mathcal{Y}^i_{new}$ from round $i\in\{1, 2, \dots, r\}$.}
	We then compute the weight $w^r_{v,y} \in \mathcal{W}^r_{new}$ for each edge $e^r_{v,y} \in \mathcal{E}^r_{new}$ based on the keyword's importance within the text corpus and keyword relatedness to the label $y$. 
	At last, we merge all graphs $\mathcal{G}^r_{new}(\mathcal{V}^r_{new}, \mathcal{E}^r_{new},\mathcal{W}^r_{new})$ at each round $r$ as the full graph $\mathcal{G}^r(\mathcal{V}^r, \mathcal{E}^r,\mathcal{W}^r)$.
	More details can refer to Section~\ref{sec:gc}.

	\noindent\textbf{Part 2: Graph Retrieval.} 
	The Graph Retrieval process maps the extracted keyword nodes $\mathcal{V}^t$ from the query text $t$ to the constructed weighted graph $\mathcal{G}^r$ 
	and generates a minimum-cost spanning tree on $\mathcal{G}^r$ that includes all mapped keywords in $\mathcal{V}^t$. 
	From this minimum-cost spanning tree (MST), the candidate label set $\hat{\mathcal{Y}}_t^r$ at the current round $r$ is obtained, 
	which is a reduced subset of the original target label set $\mathcal{Y}^r$ at the current round $r$. 
	As this MST solely depends on the graph itself and the text extracted keywords, we can eliminate any human-defined thresholds and achieve adaptive retrieval. 
	For more details, please  refer to Section~\ref{sec:gr}.


	\noindent\textbf{Part 3: Classification \& Online Graph Indexing.} 
	In the Classification part, we create LLM input with the query text $t$, the candidate labels $\hat{\mathcal{Y}}_t^r$ retrieved from the weighted graph $\mathcal{G}^r$, and descriptions $\mathcal{K}_{y_i}$ associated with each candidate label $y_i \in \hat{\mathcal{Y}}_t^r$ if 
	available. 
	The LLM will carry out classification with this input. 
	After classification, the online graph indexing procedure dynamically adds keywords $\mathcal{V}_{notexist}^t$, which are extracted from the query text $t$ but are not in the existing graph $\mathcal{G}^r$, as new graph nodes to further enrich the graph. 
	We then 
	assign edge weights based on these newly added keywords' importance within the text corpus and their relatedness to the predicted label $y^*_t$ for the query text $t$.
	By penalizing the edges between nodes and their semantically irrelevant labels, GORAG can reduce the error introduced by inaccurate classification.
	The Online Graph Indexing procedure improves the model classification ability by updating the graph in real-time with new keywords. 
	For more details, please refer to Section~\ref{sec:oi}.
	Due to space limitations, we put the time and space complexity analysis of GORAG to \autoref{sec:complexity}.
	\vspace{-5px}

\subsection{Part 1: Offline Graph Indexing} \label{sec:gc}
In this subsection, we introduce the graph construction procedure of GORAG, which includes a novel edge-weighting mechanism to address the uniform indexing issue. 
The pseudo code of graph construction is shown in Algorithm \autoref{alg:graph_construction} in \autoref{sec:pseduo_code}.

For each round $r$,
given its respective labeled training texts $\mathcal{D}^r=\cup_{y_i \in \mathcal{Y}^r}\{t_j,y_i\}_{j=1}^k$, 
GORAG first extracts text keywords $\mathcal{K}^r$ from 
$\mathcal{D}^r$. 
Specifically, it
calculates the normalized TFIDF score \cite{tfidf} as the correlation score $CS(v,t)$ for texts' each word and phrase with less than 3 words \cite{wang2025agrag}. 
\vspace{-3px}
\begin{equation}\label{eq:tfidf_ke}
	\vspace{-4px}
	CS(v,t)=\frac{count(v,t)}{|t|log(|\mathcal{T}^r|+1)}\times log\frac{|\mathcal{T}^r|+1}{|t_j:v\in t_j, \ t_j \in \mathcal{T}^r|+1}, 
\end{equation}
where $ \mathcal{T}^r=\cup_{i=1}^r \mathcal{T}^i$ denotes all training and query texts seen so far, 
$count(v, t)$ is the number of times
that the term $v$ appears in the text $t$, 
and $|t_j : v \in t_j, \ t_j \in \mathcal{T}^r|$ denotes the number of texts in the corpus $\mathcal{T}^r$ that contain the keyword $v$. 
Here, we will extract words or phrases with $CS(v,t)>\tau$ as our keyword. 
We set $\tau=0.5$ following the keyword extraction setting in \cite{wang2025agrag}.
These keywords serve as graph nodes, and we denote them as keyword nodes.

Also, GORAG incorporates new labels $\mathcal{Y}^r_{new}=\mathcal{Y}^r\setminus\mathcal{Y}^{r-1}$ at the $r$-th round as graph nodes, denoted as label nodes, where $\mathcal{Y}^r$ is the total label set at round $r$. Hence, the graph node set $\mathcal{V}^r_{new}$ at the $r$-th round can be obtained as follows. 
\begin{equation}
	\vspace{-5px}
	\mathcal{V}_{new}^r=\mathcal{K}^r \cup\mathcal{Y}^r_{new}.
\end{equation}
Then, GORAG adds edges between each keyword node $v \in \mathcal{K}^r$ to its corresponding label node $y \in \mathcal{Y}^r$, indicating that the keyword $v$ appears in texts associated with the label $y$.

Next, 
considering the keywords $\mathcal{K}^r$ from the labeled text are not uniformly related to the text's label in $\mathcal{Y}^r$, 
we apply an edge weighting mechanism to assign a weight $w_{v,y}$ to each keyword-label edge between any two nodes $v$ and $y$.
This weight can reflect the correlation between keywords and each label.
Firstly, 
we apply the correlation score $CS(v,t)$ assigned by Equation~\eqref{eq:tfidf_ke} to measure the importance and relatedness of a particular keyword $v \in \mathcal{K}^r$ w.r.t. 
the text $t$ 
it extracted from.

As the keyword $v$ can be extracted from multiple texts with different labels in different rounds, the final edge weight $w_{v,y}^r$ of edge $e_{v,y}$ 
at the $r$-th round is calculated as an average of all weights from all seen texts with label $y$:
\vspace{-5px}
\begin{equation}\label{eq:edge_weigting}
	w_{v,y}^r=\frac{\sum_{t_j \in \mathcal{T}^{r,v,y} }1-CS(v,t_j)}{
|\mathcal{T}^{r,v,y} |
	}, 
\end{equation}
where $\mathcal{T}^{r,v,y}  = \{t_j|   v\in t_j \land t_j \in \mathcal{T}^{r,y}\}$ is the text that contains keyword $v$ and labeled $y$. 
We denote the generated weighted graph for $\mathcal{K}_{new}^r$ and $\mathcal{Y}_{new}^r$ at the $r$-th round as $\mathcal{G}^r_{new}(\mathcal{V}_{new}^r,\mathcal{E}_{new}^r, \mathcal{W}_{new}^r)$, 
where $\mathcal{V}^r_{new}$, $\mathcal{E}^r_{new}$, and $\mathcal{W}^r_{new}$ denotes the node, edge and edge weight set respectively.
Lastly, $\mathcal{G}^r_{new}$ will be merged into the graph from previous rounds $\mathcal{G}^{r-1}$ to form the full graph $\mathcal{G}^r$ at round $r$: 
\begin{align}
	&\mathcal{G}^r(\mathcal{V}^r,\mathcal{E}^r,\mathcal{W}^r),\\ \mathcal{V}^{r}=\mathcal{V}^{r}_{new}\cup \mathcal{V}^{r-1}, &\ \mathcal{E}^{r}=\mathcal{E}^{r}_{new}\cup \mathcal{E}^{r-1},\  \mathcal{W}^{r}=\mathcal{W}^{r}_{new}\cup \mathcal{W}^{r-1}.\notag 
\end{align}
Particularly, we define $\mathcal{G}^0=\emptyset$ and $\ r\geq 1$. 
Also, to guarantee the graph connectivity of the resulting graph, 
we add edges between every new label node $y_{n}\in \mathcal{Y}_{new}^r$, 
and each old label node $y_{o}\in \mathcal{Y}^{r-1}$, 
the edge weight $w_{y_n,y_o}^r$ of edge $e_{y_n,y_o}$ at round $r$ is defined as the weighted average of all edge weights that link keyword nodes with label node $y_n$ or $y_o$: 
\vspace{-3px}
\begin{equation}
	\vspace{-5px}
	\mathcal{M}_y^r=\{v\ |\ v\in \mathcal{N}^r(y)\ \land\ v\notin \mathcal{Y}^r\}
	\vspace{-5px}
\end{equation}

\begin{equation}
	\vspace{-3px}
	w_{y_n,y_o}^r=\frac{1}{2}\Big(\frac{\sum_{v\in \mathcal{M}_{y_n}^r}w_{v,y_n}^r}{2\times|\mathcal{M}_{y_n}^r|}+\frac{\sum_{v\in \mathcal{M}_{y_o}^r}w_{v,y_o}^r}{2\times|\mathcal{M}_{y_o}^r|}\Big),
\end{equation}

\noindent where $\mathcal{N}^r(y)$ denote the neighbor nodes of label node $y$ in $\mathcal{G}^r$.

After merge, the graph $\mathcal{G}^r$ would be used for future retrieval, and be further updated by GORAG's online indexing mechanism.

\subsection{Part 2: Graph Retrieval}\label{sec:gr}

	In this subsection, we introduce the graph retrieval procedure of GORAG. With the graph constructed in Part 1, GORAG can adaptively retrieve a set of candidate class labels with keywords extracted from query texts without any human-defined thresholds, addressing the threshold dependent issue.

To begin with, GORAG extracts keywords $\mathcal{V}_{test}^t$ for each query text $t\in \mathcal{T}_{test}$ in the same manner with Equation~\eqref{eq:tfidf_ke}, then, $\mathcal{V}_{test}^t$ would be split into two subsets: 
\begin{equation}
	\mathcal{V}_{exist}^t \cup \mathcal{V}_{notexist}^t  =\mathcal{V}_{test}^t,\ \mathcal{V}_{exist}^t \cap \mathcal{V}_{notexist}^t  = \emptyset, 
\end{equation}
where $\mathcal{V}_{exist}^t $ and $\mathcal{V}_{notexist}^t$ denotes the keywords in $\mathcal{V}_{test}^t$ that already exist and not yet exist in $\mathcal{G}^r$ at current round $r$ respectively. Later, $\mathcal{V}_{exist}^t $ would be applied for achieving adaptive retrieval, and $\mathcal{V}_{notexist}^t$ would be applied for online indexing to further enrich the constructed graph (Further illustrated in Section~\ref{sec:oi}).

To achieve the adaptive retrieval, GORAG tries to find the minimum cost spanning tree (MST) that contain all keyword nodes within $\mathcal{V}_{exist}^t$, 
and then retrieves all label nodes within the generated MST as candidate labels.
The intuition behind this approach is that an MST spans the entire graph to cover all given nodes with the smallest possible spanning cost. Consequently, label nodes within the generated MST can be considered important for demonstrating the features of the given keyword node set. As the generation of an MST is a classical combinatorial optimization problem with an optimal solution determined solely by the set of given nodes. 
By generating the MST and retrieving all label nodes within it, we eliminate the need for any human-defined thresholds, achieving adaptive retrieval.

\begin{table*}[ht!]
	\centering   
	\caption{Statistics of the tested datasets, where we divide their original labels into multiple rounds. We achieve a balanced testing set on the IFS-Rel datasets by assigning each label 40 testing samples.}
	\label{tab:data_stat}
	\vspace{-10px}
	\resizebox{\linewidth}{!}{%
		\begin{tabular}{l|c|c|c|c|c|c|c|c|c|c} 
			\hline\hline
			\multirow{2}{*}{\textbf{Dataset}} & \multirow{2}{*}{\textbf{Avg. Text Token}} & \multicolumn{2}{c|}{\textbf{R1}} & \multicolumn{2}{c|}{\textbf{R2}} & \multicolumn{2}{c|}{\textbf{R3}} & \multicolumn{2}{c|}{\textbf{R4}} & \textbf{Total}  \\ 
			\cline{3-11}
			&                                           & Testing data & Label \#          & Testing data & Label \#          & Testing data & Label \#          & Testing data & Label \#          & Testing data    \\ 
			\hline
			CAD \cite{vidgen2021introducing}      & 154                                       & 240        & 4                & 60        & 3                & 286        & 3                & 231        & 3                & 817           \\ \hline
			WOS \cite{kowsari2017hdltex}      & 200                                       & 2,417        & 32                & 2,842        & 53                & 2,251        & 30                & 1,886        & 18                & 9,396           \\ 
			\hline
			IFS-Rel \cite{xia-etal-2021-incremental}      & 105                                       & 640        & 16               & 640        & 16                & 640        & 16                & 640        & 16                & 2,560           \\ 
			\hline\hline
		\end{tabular}
	}
	\vspace{-10px}
\end{table*}
\begin{definition} [Adaptive Candidate Label Generation Problem]  
	Given an undirected weighted connected graph $\mathcal{G}^r(\mathcal{V}^r,\mathcal{E}^r,\mathcal{W}^r)$, a set of keywords $\mathcal{V}_{exist}^t$ extracted from text $t$ that can be mapped to nodes in $\mathcal{V}^r$, and the target label set $\mathcal{Y}^r\in \mathcal{V}^r$ at the $r$-th round, our target is to find a set of labels nodes ${\mathcal{Y}}^r_t \in {\mathcal{Y}}^r$. 
	Firstly, we identify a subgraph $\mathcal{G}^r_t(\mathcal{V}^r_t, \mathcal{E}^r_t, \mathcal{W}^r_t)$ of $\mathcal{G}^r$ by minimizing the edge weight sum as follows.
	\vspace{-10px}
	\begin{align}
	&	\min_{e^r_{u.v} \in \mathcal{E}^r_t}w^r_{u,v} \notag\\
			s.t. \  &v \in \mathcal{V}^r_t, \forall v \in \mathcal{V}_{exist}^t 
	\end{align}
Then, since the subgraph node  $\mathcal{V}^r_t$ set contains both keyword nodes and label nodes,
we define the candidate label set $\hat{\mathcal{Y}}^r_t \in \mathcal{V}^r_t$ as the label nodes in the generated sub-graph for the text $t$ at round $r$.
\end{definition}
\vspace{-5px}
 Due to space limitations, we put the proof of the Adaptive Candidate Label Generation Problem to be NP-hard to \autoref{sec:proof}. 

As the problem is NP-hard, it is infeasible to obtain the optimal result in polynomial time.
Therefore, 
to solve this problem, we propose a greedy algorithm based on the Melhorn's algorithm \cite{mehlhorn1988faster}. Our algorithm generates the candidate label set $\hat{\mathcal{Y}}^r_t \subseteq \mathcal{Y}^r$ for text $t$ at the $r$-th round.
The pseudo code 
of our algorithm is shown 
in Algorithm \autoref{alg:mehlhorn} in \autoref{sec:pseduo_code}
, following the Melhorn's algorithm, our algorithm can also achieve 2-Approximate.

Firstly, we calculate the shortest path between each keyword node $v$ in $\mathcal{V}_{exist}^t$ to all other nodes in $\mathcal{G}^r$ by calculating the minimum spanning tree $MST$ of $\mathcal{G}^r$ 
Here, the keyword nodes $v$ in $\mathcal{V}_{exist}^t$ serve as terminal nodes that determines the final generated candidate labels w.r.t. to the weighted graph $\mathcal{G}^r$.
Secondly, we create a new auxiliary graph $H$ where the edges represent the shortest paths between the closest terminal nodes. 
Thirdly, we construct the minimum spanning tree $MST'$ of the auxiliary graph $H$ 
, and then add the shortest paths between each two nodes in $MST'$ to the Steiner Tree $ST$.
Lastly our candidate labels $\hat{\mathcal{Y}}^r_t$ for text $t$ are calculated as the intersection of $ST$ and all target labels $\mathcal{Y}^r$ at round $r$: $ST\cap\mathcal{Y}^r$.
\subsection{Part 3: Classification \& Online Indexing}\label{sec:oi}
In this subsection, we first introduce GORAG's classification and online graph indexing procedure, where it classifies the texts and further enriches the weighted graph by adding query text information to the created graph, addressing the narrow source issue.
To begin with, we introduce how GORAG performs text classification based the LLM. 
Specifically,
given each unlabeled query text $t$, GORAG utilizes LLM to predict its class label $y^*$ by constructing an input prompt $c_t$ as follows: 
\begin{align}
	c_t=Concat(t, \mathcal{K}^t, \hat{\mathcal{Y}}^r_t, \mathcal{K}_{\hat{\mathcal{Y}}^r_t}). 
\end{align}
 $c_t$ can be considered as the text concatenation of the extracted keywords $\mathcal{K}^t$ from the text $t$, and the candidate labels $\hat{\mathcal{Y}}^r_t$ obtained by Algorithm~\ref{alg:mehlhorn} in \autoref{sec:pseduo_code}, we omit the more complex input structures as they rarely improve performance, but lead to higher costs \cite{zhou2023lima,liu2024lost}.
 Also, the $\mathcal{K}_{\hat{\mathcal{Y}}^r_t}=\{\mathcal{K}_{y_i}\}_{i=1}^{|\hat{\mathcal{Y}}^r_t|}$
 and $\mathcal{K}_{y_i}$ are the representative keywords of each label $y_i\in \hat{\mathcal{Y}}^r_t$. 
\begin{equation}
	\mathcal{K}_{y_i}=LLM(p_{gen},\mathcal{D}^r), 
\end{equation}
which were generated by LLM with the label's semantic label name if available. 
Next, with a classification instruction prompt $p_{classify}$. 
\begin{equation}
	y^*_t=LLM(p_{classify}, c_t), 
\end{equation}
the LLM would try to select the best-suited label $y^*_t \in \mathcal{Y}^r$ to annotate the text $t$. 


To fully leverage the text-extracted keywords, GORAG utilizes an online indexing mechanism to incrementally update keywords that do not yet exist in the weighted graph $\mathcal{G}^r$ at the $r$-th round to $\mathcal{G}^r$ based on the text $t$'s predicted label $y^*_t$. 
To be specific, each keyword node $v\in \mathcal{V}_{notexist}^t$ would be added to the original graph's node set $\mathcal{V}^r$ and be assigned with an edge $e_{v,y}$ connecting it with the text $t$'s predicted label $y_t^*$:  
\vspace{-3px}
\begin{equation}
	\vspace{-3px}
	\mathcal{V}^r = \mathcal{V}^r \cup \mathcal{V}_{notexist}^t,\ \mathcal{E}^r = \mathcal{E}^r \cup \mathcal{E}_{oi}^t, 
\end{equation}
where $\mathcal{E}_{oi}^t=\{e_{v,y_t^*}| v\in \mathcal{V}_{notexist}^t\}$ denotes the set of all newly assigned edges $e_{v,y}$ between keyword node $v$ and its predicted label $y_t^*$, for these newly assigned edges, their weight is calculated with the edge weighting mechanism illustrated in Equation~\eqref{eq:edge_weigting}.
The edge weights can reflect the semantic relevance of keywords to labels, and reduce the noise introduced by inaccurate LLM classification.

%% file: section/sec-experiment_results.tex
\vspace{-5px}
\section{Experiments}\label{sec:experiment}
\begin{table*}[t]
	\centering
	\renewcommand{\arraystretch}{1.1}
	\caption{Classification accuracy on CAD, WOS, and IFS-Rel benchmarks.
	}
	\label{tab:main_exp}
	\vspace{-10px}
	\setlength\tabcolsep{3pt}
	\resizebox{\linewidth}{!}{
		\begin{tabular}{l|l|l|ccc|ccc|ccc|ccc} 
			\hline\hline
			\multirow{2}{*}{\textbf{Dataset}} & \multirow{2}{*}{\textbf{Category}} & \multirow{2}{*}{\textbf{Model}} &
			\multicolumn{3}{c|}{\textbf{Round 1 }}                    & \multicolumn{3}{c|}{\textbf{Round 2 }}                    & \multicolumn{3}{c|}{\textbf{Round 3 }}                    & \multicolumn{3}{c}{\textbf{Round 4 }}                      \\ 
			\cline{4-15}
			&                                     &                                  & \textbf{1-shot} & \textbf{5-shot} & \textbf{10-shot} & \textbf{1-shot} & \textbf{5-shot} & \textbf{10-shot} & \textbf{1-shot} & \textbf{5-shot} & \textbf{10-shot} & \textbf{1-shot} & \textbf{5-shot} & \textbf{10-shot}  \\ 
			\hline
			\multirow{7}{*}{\textbf{CAD}}     & \multirow{1}{*}{NN-based}           & Entailment    \cite{xia-etal-2021-incremental}                   &     0.2125  &    0.2000  &    0.1208    &     0.1700      &   0.1060    &      0.1634   &   0.1177   &   0.0836      &    0.0785     &  0.0857       &   0.0700       &    0.0529       \\
			\cline{2-15}
			& \multirow{2}{*}{Long Context RAG}   & NaiveRAG      \cite{gao2023retrieval}                      &     0.1875  &     0.1375 &  0.0958      &  0.1500         &       0.1100 &     0.0766    &  0.0802    &      0.0563  &   0.0392      &   0.0600      &   0.0405       &    0.0281     \\
			
			&   & LongLLMLingua        \cite{jiang-etal-2024-longllmlingua}              &     0.1667  &   0.2417   &     0.2417   &        0.1709   &    0.2167   &     0.2200    & 0.1370     &  0.1587       &    0.1552     &   0.1092      &    0.1359      &     0.1285    \\ \cline{2-15}
			&         \multirow{4}{*}{Graph-based RAG}                              & GraphRAG       \cite{edge2024local}                  &   \underline{0.3958}    &      \underline{0.3625} &     {0.3292}   &    \underline{0.3167}       &    \underline{0.2933}   &      \underline{0.2833}   & 0.1587     &   \underline{0.1621}      &   \underline{0.1706}      &     0.1200    &   \underline{0.1236}       &        \underline{0.1322}    \\
			&                                     & LightRAG      \cite{guo2024lightrag}                     &   0.3875	    &   0.3583   &      \underline{0.3373}   &      0.3133     &     0.2867 &   0.2867      &  0.1638    &      0.1502    &      0.1571   &     0.1175    &   0.1102       &    0.1188      \\
			&                                     & HippoRAG2     \cite{gutiérrez2025ragmemory}               &    0.3667   &     0.2250 &    0.1792    &   0.2953        &    0.2000    &    0.1458    &     \underline{0.1820} &     0.1032    &    0.0746     &    \underline{0.1435}     &         0.0668 &    0.0535    \\ 
			&                                     & \textbf{GORAG}            &   \textbf{0.6125}    &   \textbf{0.6125}      &     \textbf{0.6125}       &  \textbf{ 0.5867}    &  \textbf{0.5833}     &    \textbf{0.5882}       &     \textbf{0.3993}    &   \textbf{0.4010}      &     \textbf{0.4001}        &       \textbf{0.3133}  &    \textbf{0.3146}        &        \textbf{03139}            \\ \hline\hline
			\multirow{7}{*}{\textbf{WOS}}     & \multirow{1}{*}{NN-based}           & Entailment    \cite{xia-etal-2021-incremental}                   & {0.3695}          & {0.3823}          & {0.4187}           & \underline{0.3994}          & \underline{0.4471}          & \underline{0.4222}           & \underline{0.4510}          & \underline{0.4857}          & \underline{0.4787}           & \underline{0.4030}          & \underline{0.4387}          & \underline{0.4442}            \\
			\cline{2-15}
			& \multirow{2}{*}{Long Context RAG}   & NaiveRAG      \cite{gao2023retrieval}                   & 0.3885          & {0.3904}          & 0.3897           & 0.2267          & 0.2154          & 0.2187           & 0.1821          & 0.1475          & 0.1799           & 0.1653          & 0.1556          & 0.1649            \\
			
			&   & LongLLMLingua        \cite{jiang-etal-2024-longllmlingua}            & 0.3806          &          0.3823       &       0.3901           & 0.2155          &          0.2202       &      0.2198            & 0.1770          &        0.1567        &       0.1608        & 0.1468          &         0.1382       &       0.1493            \\ \cline{2-15}
			&                 \multirow{4}{*}{Graph-based RAG}                      & GraphRAG       \cite{edge2024local}                  & 0.3852          & 0.3897          & 0.3906           & 0.2213          & 0.2197          & 0.2219           & 0.1816          & 0.1770          & 0.1786           & 0.1641          & 0.1634          & 0.1625            \\
			&                                     & LightRAG      \cite{guo2024lightrag}                   & {0.3930}          & 0.3806          & 0.3815           & 0.2202          & 0.2216          & 0.2145           & 0.1743          & 0.1767          & 0.1799           & 0.1625          & 0.1626          & 0.1632            \\
			&                                     & HippoRAG2    \cite{gutiérrez2025ragmemory}               &  \underline{0.4133}    &     \underline{0.4245}   &   \underline{0.4208}     &  0.2552    &    0.2656  &     0.2666    &    0.2205  &  0.2210  &   0.2242   &  0.1906     &   0.1860 &      0.1915    \\
			&                                     & \textbf{GORAG}                   &     \textbf{0.4866}            &        \textbf{0.4990}            &    \textbf{0.5134}              &       \textbf{0.4790}           &        \textbf{0.5109}            &        \textbf{0.5284}          &            \textbf{0.4736}      &       \textbf{0.4996}          &     \textbf{0.5234}              &          \textbf{0.4420}        &         \textbf{0.4717}         &       \textbf{0.4929}             \\ 
			\hline\hline
			\multirow{7}{*}{\textbf{IFS-Rel} } & \multirow{1}{*}{NN-based}           & Entailment          \cite{xia-etal-2021-incremental}           & 0.3391          & {0.6046}          & {0.5859}           & {0.4008}          & {0.5516}          & \underline{0.5828}           & {0.3061}          & {0.4703}          & {0.5193}           & {0.2971}          & {0.4395}          & {0.4895}            \\\cline{2-15}
			& \multirow{2}{*}{Long Context RAG}   & NaiveRAG        \cite{gao2023retrieval}               &{0.7343}          & 0.7344          & 0.7406           & 0.5394          & 0.5406          & 0.5469           & \underline{0.4719}          & 0.4719          & 0.4688           & 0.5233          & {0.4938}         & 0.5000            \\
		
			&    & LongLLMLingua       \cite{jiang-etal-2024-longllmlingua}              & \underline{0.8750}          &        \underline{0.8906}         & \underline{0.9219}           & \underline{0.5734}          &         \underline{0.5922}        & 0.5806           & 0.4703          &   \underline{0.4741}              & \underline{0.5225}           & \underline{0.5819}          &      0.4903           & \underline{0.5038}            \\	\cline{2-15}
			&             \multirow{4}{*}{Graph-based RAG}                         & GraphRAG      \cite{edge2024local}                     & 0.6406     &    0.7813   &   0.8125    &   0.4406    &     0.4688    &   0.4719     &     0.3828   &   0.3734    &     0.3141      &  0.4297   &    0.4422    &   0.3922           \\
			&                                     & LightRAG    \cite{guo2024lightrag}               & 0.7578     &    0.8281   &    0.8594    &  0.4906    &  0.5155    &    0.5156     &  0.3802    &  0.4010  &  0.4427    &   0.4531    & 0.4804   &  0.4610        \\
			&                                     & HippoRAG2    \cite{gutiérrez2025ragmemory}               & 0.7313     &     0.7375  &     0.7406   &    0.4875  &    0.5016  &    0.4938     &    0.4297  &   0.4328 & 0.4438     &     0.4984  &   \underline{0.5008} &    0.4953      \\
			&                                     & \textbf{GORAG}                                      & \textbf{0.9688}   & \textbf{0.9688}   &   \textbf{0.9688}&\textbf{0.5938}  &\textbf{0.6172} & \textbf{0.6250} & \textbf{0.4844}  & \textbf{0.4948} &\textbf{0.5344}  &\textbf{0.6165} & \textbf{0.4988}  & \textbf{0.5469}  \\
			\hline\hline
		\end{tabular}
	}
	\vspace{-10px}
\end{table*}

In this section,
we evaluate GORAG's performance on the DFSTC task.
Section~\ref{ssec:exp:setting} introduces the experimental setup, while Section~\ref{ssec:exp:main} presents results on effectiveness and efficiency. Section~\ref{ssec:exp:ablation} covers ablation studies, and Section~\ref{ssec:exp:gen} discusses generalization. 
	Due to space limitation, we put experiment details, such as hardware settings and case study in \autoref{sec:exp_detail}.
\subsection{Experiment Settings}\label{ssec:exp:setting}
 
\subsubsection{\textbf{Few-shot setting}} 
In this paper, we employ 1-shot, 5-shot, and 10-shot settings for few-shot training, where each setting corresponds to using 1, 5, and 10 labeled training samples per class, respectively. 
We omit experiments with more than 10 labeled samples per class, as for WOS dataset, there are already over 1300 labeled training data under 10-shot setting. 
 
\subsubsection{\textbf{Datasets}}
We select the CAD \cite{vidgen2021introducing}, WOS \cite{kowsari2017hdltex}, 
and IFS-Rel \cite{xia-etal-2021-incremental} benchmark to test GORAG's performance. 
The texts in CAD are social media posts that potentially have different kinds of hate speech and abuse. 
The texts in the 
IFS-Rel benchmark are segments from web news, and their classifications help detect misinformation. 
The texts from the WOS benchmark are abstracts or introductions from academic papers, as the WOS dataset contains a vast number of candidate labels and has a large word count for each text; these characteristics make it a challenging benchmark for text classification models. The statistics of these benchmarks are shown in  \autoref{tab:data_stat}.

Due to the space limitations, we mainly display our experinents on the 4 rounds split version of all datasets in \autoref{tab:main_exp}, and further experiment results of 6 and 8 rounds are shown in \autoref{fig:robust}.

\noindent\underline{\textit{\textbf{Ethical Disclaimer.}}}
This research involves sensitive content. All datasets are publicly available, and our study adheres to institutional ethical guidelines to ensure that no personally identifiable information is disclosed. Examples shown in this paper are for illustrative purposes only and do not reflect the authors' views.
\subsubsection{\textbf{Evaluation Metrics}}
In this paper, aligning with previous DFSTC models \cite{xia-etal-2021-incremental, meng2020text}, 
we use classification accuracy as the evaluation metric. 
Given that LLMs can generate arbitrary outputs that may not precisely match the provided labels, we consider a classification correct only if the LLM's output exactly matches the ground-truth label name or label number.
\subsubsection{\textbf{Baselines}}
	We compare six baselines from three types as follows.
	The details of selected baselines are in Appendix~\ref{sec:baselines}. 
	\begin{itemize}[leftmargin=*]
		\item 	 \textbf{One NN-based model: } \textit{Entailment} \cite{xia-etal-2021-incremental}, which converts the task to binary classification, and creates entailment pairs with the labeled data for training. 
		
		\item 	 \textbf{Two Long Context RAG models:  }
		\textit{NaiveRAG} \cite{gao2023retrieval} splits few-shot labeled data into chunks and retrieves the most query similar chunks to augment LLM classification.
		\textit{LongLLMLingua} \cite{jiang-etal-2024-longllmlingua} further compresses the split chunks of NaiveRAG by filtering out chunk tokens regarding their LLM generation perplexity.
		
		\item 	 \textbf{Three Graph-based RAG models: }
		\textit{GraphRAG} \cite{edge2024local} extracts entities and relations from few-shot training data, and uses LLM to summarize graph neighborhoods as communities for retrieval. 
		\textit{LightRAG} \cite{guo2024lightrag} drops the community-related procedure and solely retrieves the 1-hop neighborhood of the unlabeled text mapped entities on the graph.
		\textit{HippoRAG2} \cite{gutiérrez2025ragmemory} retrieves the Top-k candidate labels and their training texts based on their Personalized PageRank (PPR) score on the graph.
		
	\end{itemize}

Furthermore, as solely put the labeled text within the LLM context can greatly increase costs and affect LLM performance
\cite{chen2023dense,jiang-etal-2024-longllmlingua,liu2024lost}, 
we compare the performance of solely using LLM for classification with GORAG under 0-shot setting in \autoref{tab:0_shot}. 

\subsubsection{\textbf{Hyperparameter Settings}}


Following the exact setting in Entailment's original paper \cite{xia-etal-2021-incremental}, we train the Entailment model with the RoBERTa-large PLM for 5 epochs with a batch size of 16, a learning rate of $1 \times 10^{-6}$, and the Adam optimizer~\cite{kingma2014adam}.
For LongLLMLingua, we test the compression rate within ${0.75, 0.8, 0.85}$ and select 0.8, as it achieves the best overall classification accuracy.
For all RAG-based baselines, we set all their hyperparameter the same as in their original paper, if not further illustrated.
For GraphRAG and LightRAG, we use their local search mode, as it achieves the highest classification accuracy on the most complex WOS dataset. 
For GraphRAG and NaiveRAG, we use the implementation from \cite{nanographrag}, which optimizes their original code and achieves better time efficiency while not affect the performance.
For all other baselines, we use their official implementations. 

\subsection{Few-shot Experiments}\label{ssec:exp:main}
%
%
%
\begin{figure}[t]
	\hspace{-11px}
	\subfigure[\footnotesize{Time cost on WOS dataset.}]{
		\label{fig:eff_wos}
		\includegraphics[width=0.5\linewidth]{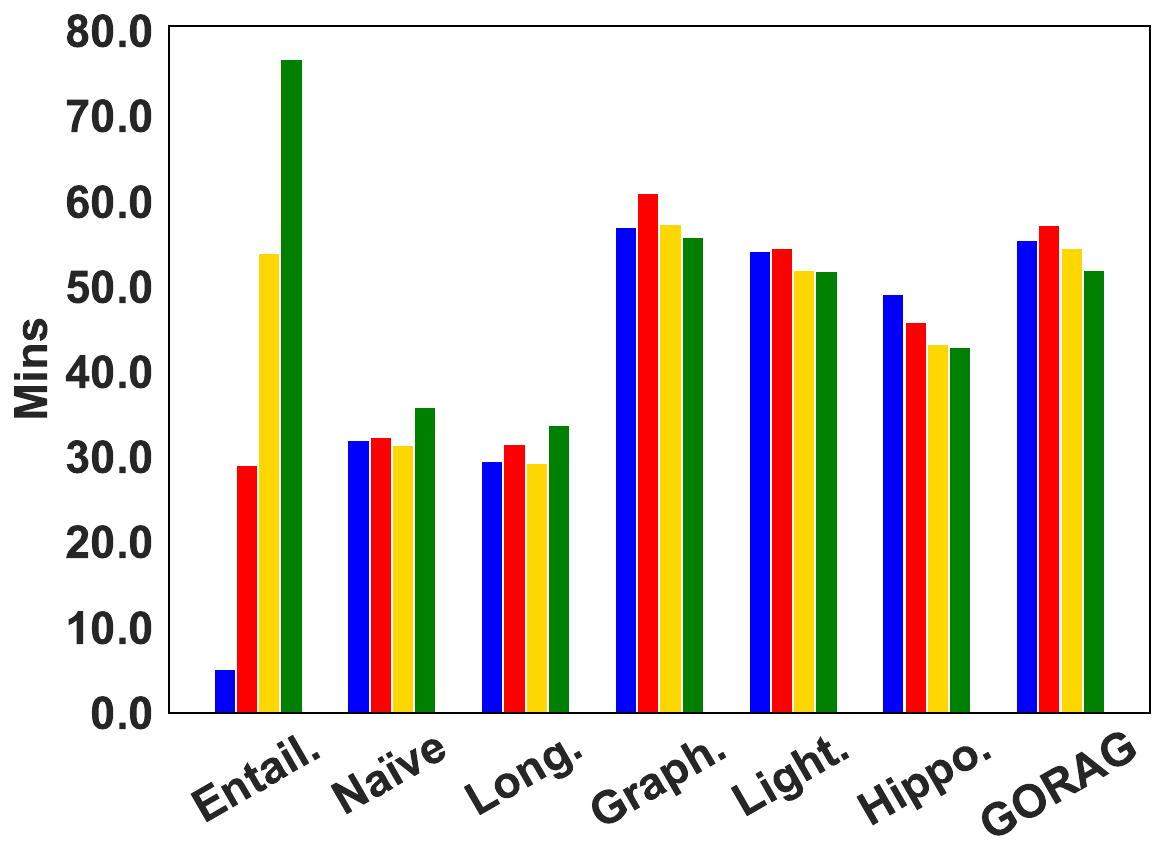}
	}
	\hspace{-10px}
	\subfigure[\footnotesize{Time cost on CAD dataset.}]{
		\label{fig:eff_cad}
		\includegraphics[width=0.5\linewidth]{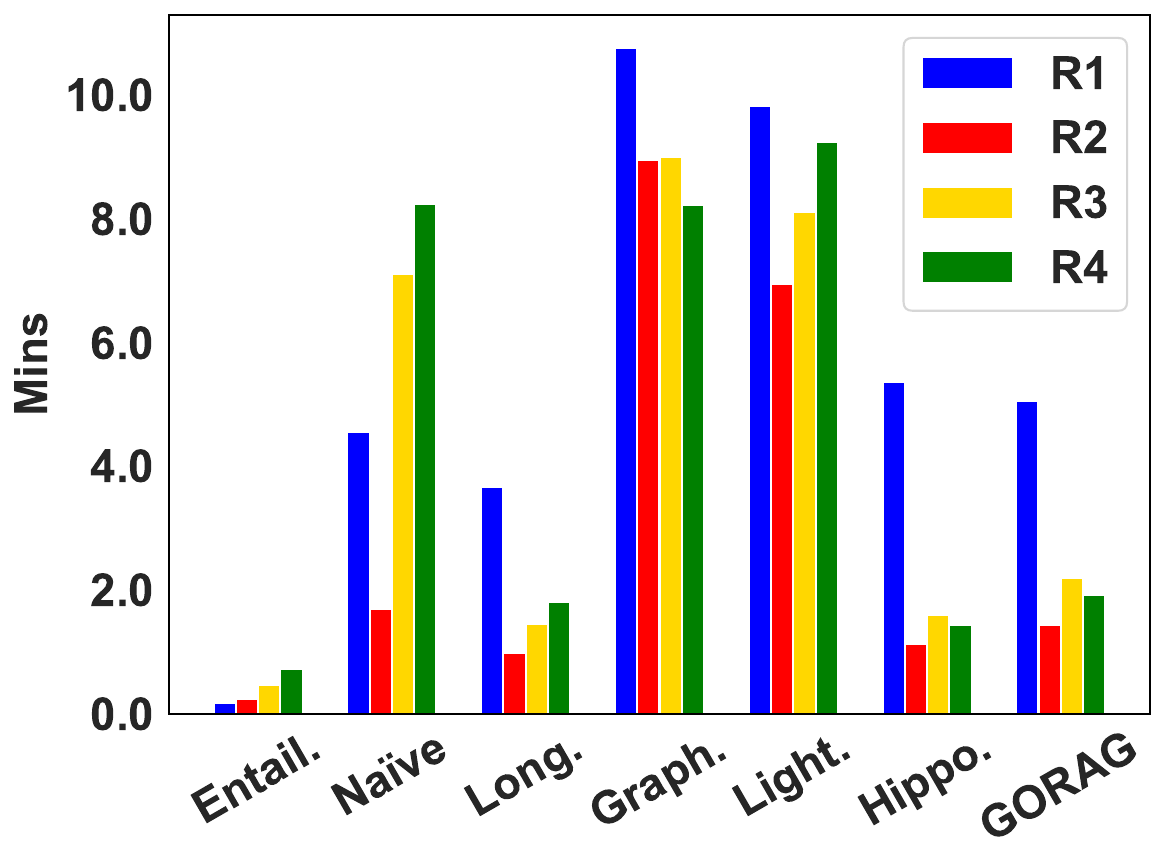}
	}
	\vspace{-10px}
	\caption{Total time cost on 2 datasets under 1-shot setting.}
	\label{fig:time_cost}
	\vspace{-10px}
\end{figure}
\subsubsection{\textbf{Effectiveness Evaluation}}\label{sec:effectiveness}
	In \autoref{tab:main_exp}, we apply Qwen2.5-7B backbone for CAD and WOS benchmark. 
	Due to cost and API risk control (failing to respond when detecting dirty words) reasons, we only use the GPT-4o backbone on the IFS-Rel news benchmark. 
As shown in \autoref{tab:main_exp}, GORAG achieves the best classification accuracy over all dynamic rounds on all benchmarks. 
Based on the experiment results, 
we have the following observations: 

\begin{itemize}[leftmargin=*]
\item Firstly, with the training data increase (R1 to R4), NN-based Entailment suffers from performance decrease on CAD datasets.
This is because the CAD datasets have an unbalanced label distribution, 
%
as noted in prior studies \cite{liu2023pre,liu2024liberating,lei2023tart}, 
NN-based models that apply data augmentation techniques are prone to overfit when the ground truth label distribution are unbalanced.



\item Secondly,
NaiveRAG performs better than Entailment in round 1 of the WOS and IFS-Rel datasets by avoiding overfitting, but its long inputs hinder its ability to adapt to dynamic label updates, leading to a significant performance drop from R1 to R4. 

%
%

\item Thirdly, 
Graph-based RAG baselines 
only slightly improve performance across datasets. They struggle with dynamic label updates due to issues like uniform indexing, threshold dependency, and narrow sources, leading to inaccurate retrievals. 


\item Lastly, 
from R1 to R4, GORAG adapts better to dynamic label updates, addresses the issues of Graph-based RAG models, and achieve better performance than all tested baselines. 

\end{itemize}
\begin{table}[t]
	\centering
	\caption{Graph incremental trends and memory usage on IFS-Rel dataset round 1 with GPT-4o LLM.}
	\label{tab:memory_cost}
	\vspace{-10px}
	\resizebox{\linewidth}{!}{%
		\begin{tabular}{l|c|c|c|c} 
			\hline\hline
			\textbf{Object} & \textbf{25\%} & \textbf{50\%} & \textbf{75\%} & \textbf{100\%}  \\ 
			\hline
			Nodes           & 1644          & 3238 (+1394)  & 4482 (+1244)  & 5299 (+817)     \\
			Edges           & 1771          & 3265 (+1494)  & 4597 (+1332)  & 5503 (+906)     \\
			Memory          & 13.3 KB       & 27.6 KB       & 37.0 KB       & 46.4 KB         \\
			\hline\hline
		\end{tabular}
	}
	\vspace{-15px}
\end{table}
\subsubsection{\textbf{Efficiency Evaluation}}\label{sec:efficiency}

To evaluate the efficiency of GORAG, we compare the total time cost of GORAG and other baselines on the 4 round split version of WOS and CAD datasets with Qwen2.5-7B backbone, the results are 
shown in Figure \autoref{fig:eff_wos} and \autoref{fig:eff_cad}. 
Although GORAG's online indexing mechanism introduces a vast amount of knowledge into the weighted graph, we can still achieve comparable or better time efficiency with other baselines. 
We also list the token usage comparison between GORAG and other baselines in \autoref{tab:token_cost}, and the graph's memory usage and incremental trends in \autoref{tab:memory_cost}, where GORAG consumes fewer tokens than current RAG baselines, and only introduce less than 50KB memory cost compared with solely applying LLMs.

\subsection{Ablation Study}\label{ssec:exp:ablation}
{\color{black}To further analyze GORAG's performance, 
we conduct an extensive ablation study under 1-shot setting with Qwen2.5-7B backbone}, as RAG models show similar performance across shot settings. Specifically, we examine how different GORAG components impact classification performance using its variants. GORAG \textit{offline} removes the online indexing mechanism. 
GORAG \textit{unit}  removes the edge weighting mechanism, every edge in this variant is assigned with weight 1. 
GORAG \textit{keyword} only uses the keyword extracted from the text to create LLM classification input, rather than the whole text. 

\begin{table}[t]
	\centering
	\caption{Token cost on IFS-Rel dataset with 1-shot setting.}
	\label{tab:token_cost}
	\vspace{-10px}
	\resizebox{\linewidth}{!}{%
		\begin{tabular}{l|c|c} 
			\hline\hline
			\textbf{Model} & \textbf{AVG Token per Query} & \textbf{AVG Cost per Query (USD)}  \\ 
			\hline
			GPT-4o         & 634.6                        & \$0.005                               \\
			LongLLMLingua  & 2771.3                       & \$0.021                               \\
			NaiveRAG       & 4537.4                       & \$0.034                               \\
			LightRAG       & 7866.3                       & \$0.059                               \\
			GraphRAG       & 9296.3                       & \$0.070                               \\
			\textbf{GORAG} & \textbf{557.5}               & \textbf{\$0.004}                      \\
			\hline\hline
		\end{tabular}
	}
	\vspace{-10px}
\end{table}
As shown in \autoref{tab:exp_abl}, firstly,
GORAG \textit{offline} achieves the worst accuracy, because it lacks a comprehensive retrieval source, making its generated candidate labels inaccurate, hence limits its performance. 
Secondly, GORAG \textit{unit} outperforms GORAG \textit{offline}, but still sub-optimal to GORAG, demonstrates the importance of modeling the varying importance between keywords and each label, which leads to accurate retrievals. 
Lastly, GORAG \textit{keywords} demonstrates the important information for text classification are mostly text keywords, as it achieves comparable or better performance than other ablated baselines. 


\begin{table}[t]
	\centering
	\caption{Ablation experiments on WOS for GORAG's variants.}
	\vspace{-10px}
	\label{tab:exp_abl}
	\resizebox{\linewidth}{!}{%
		\begin{tabular}{l|l|c|c|c|c} 
			\hline\hline
			\multirow{2}{*}{\textbf{Dataset}} & \multirow{2}{*}{\textbf{Model}} & \multicolumn{4}{c}{\textbf{Round}}                                        \\ 
			\cline{3-6}
			& & R1              & R2         & R3           & R4  \\ 
			\hline
			\multirow{5}{*}{\textbf{WOS}} 
			& GORAG \textit{offline\&unit}          &     0.2979            &     0.2287            &       0.2645          &      0.2249                        \\
			& GORAG \textit{offline}          &     0.3063            &     0.2302            &       0.2455          &      0.2156                        \\
			&  GORAG \textit{unit}             & 0.4706          & 0.4394          & 0.4407          & 0.3899                      \\
			& GORAG \textit{keyword}         & \underline{0.4746}          & \underline{0.4606}          & \underline{0.4455}          & \underline{0.4030}             \\ 
			& \textbf{GORAG}                           & \textbf{0.4862}          &  \textbf{0.4649}              &     \textbf{0.4814}            &       \textbf{0.4210}                    \\
			\hline\hline
		\end{tabular}
	}
	\vspace{-10px}
\end{table}

\subsection{More Generalization Evaluations}
\label{ssec:exp:gen}
\subsubsection{Zero-shot Experiments}
To further study GORAG's 0-shot ability, we compare GORAG's performance with the original powerful LLM GPT-4o \cite{hurst2024gpt} without RAG approaches, and some RAG baselines that are applicable to the 0-shot scenario in \autoref{tab:0_shot}. 
Due to cost constraints, we conducted experiments on the second largest IFS-Rel dataset by providing only the label names, without any labeled data to GORAG for creating the graph.
In the 0-shot setting, we input the labels and their LLM generated descriptions as system prompt for each model, 
and GORAG would extract keywords from label descriptions, without using any information from labeled texts. 
Note that the Entailment model cannot be applied in this setting, as they require labeled data for training. For other RAG-based baselines, they are considered equivalent to their base LLMs in the 0-shot setting due to a lack of external information.
According to \autoref{tab:0_shot}, GORAG can still achieve better performance than all tested baselines.


\subsubsection{Round Generalization Evaluation}\label{sec:robuestness}
	To evaluate GORAG's ability to generalize to a larger number of dynamic rounds, 
we conducted additional experiments on the 6 and 8 round split versions of the WOS dataset under 1-shot and 5-shot settings, as the WOS dataset contains more testing data and labels than CAD dataset, and more lengthy, complex texts than the IFS-Rel dataset.
In \autoref{fig:robust}, the results for rounds 1 to 4 are obtained from the 4-round split version, rounds 5 and 6 are obtained from the 6-round split version, and rounds 7 and 8 are obtained from the 8-round split version of the WOS dataset. 
This setting can ensure a sufficient amount of testing data at each round,
As illustrated in \autoref{fig:robust}, the classification accuracy of GraphRAG and NaiveRAG drops significantly from round 1 to 8, 
while GORAG maintains competitive classification accuracy as the number of rounds increases through 1 to 8, with both 1 and 5-shot setting, 
demonstrating the importance of a more precise retrieval.


\begin{table}[t]
	\centering
	\caption{0-shot experiments for GORAG and RAG baselines on IFS-Rel dataset under 0-shot setting.}
	\label{tab:0_shot}
	\vspace{-10px}
	\resizebox{\linewidth}{!}{%
		\begin{tabular}{l|c|c|c|c} 
			\hline\hline
			\textbf{Model}          & \textbf{R1}              & \textbf{R2}              & \textbf{R3}              & \textbf{R4}               \\ 
			\hline
			GPT-4o \cite{hurst2024gpt}                  & 0.9219                   & 0.5109                   & 0.4063                   & \underline{0.5936}                    \\ 
			LongLLMLingua \cite{jiang-etal-2024-longllmlingua}          & \underline{0.9281}                   & \underline{0.5297}                   & \underline{0.4234}                   & 0.5806                    \\
			LightRAG \cite{guo2024lightrag}                & 0.6718                   & 0.3047                   & 0.2135                   & 0.0977                    \\ 
			GraphRAG \cite{edge2024local}                & 0.5625                   & 0.1875                   & 0.1458                   & 0.0938                    \\
			\textbf{\textbf{GORAG}} & \textbf{\textbf{0.9328}} & \textbf{\textbf{0.5406}} & \textbf{\textbf{0.4375}} & \textbf{\textbf{0.6063}}  \\
			\hline\hline
		\end{tabular}
	}
	\vspace{-10px}
\end{table}
\begin{figure}[t]
	\centering
	\subfigure[1-shot setting.]{
		\label{fig:robust_1shot}
		\includegraphics[width=0.49\linewidth]{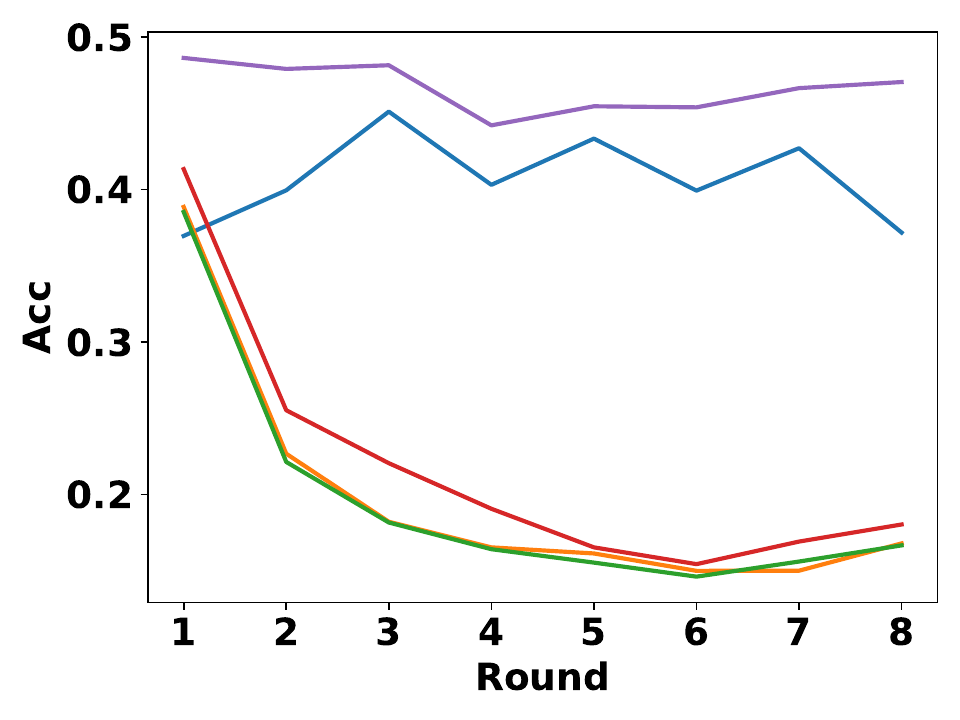}
	}
	\hspace{-9px}
	\subfigure[{5-shot setting.}]{
		\label{fig:robust_5shot}
		\includegraphics[width=0.49\linewidth]{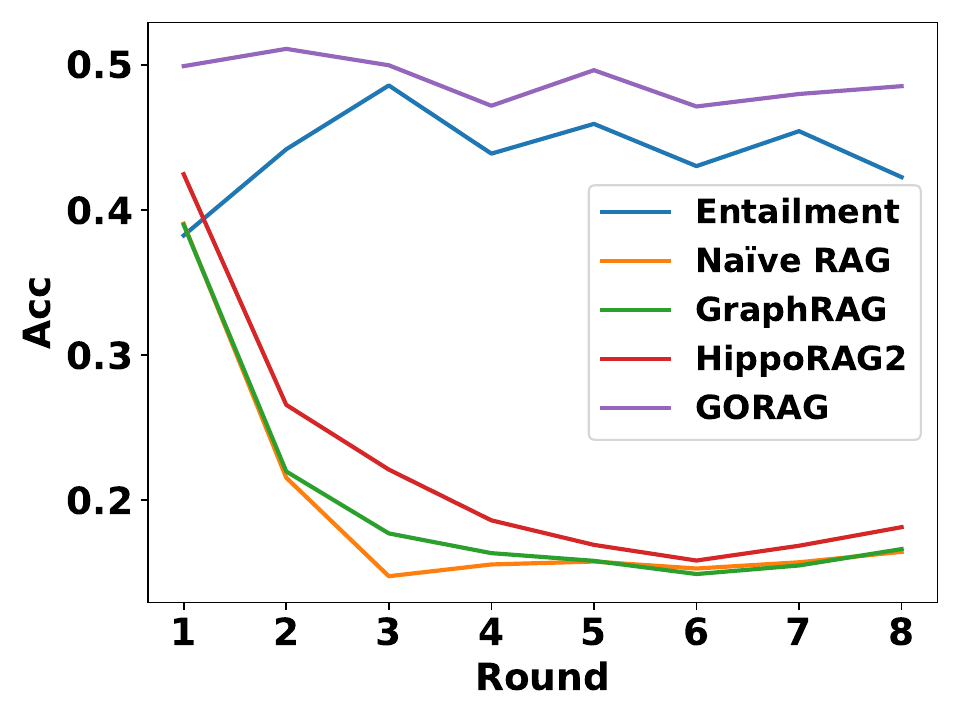}
	}
	\vspace{-15px}
	\caption{Experiment result of WOS dataset under 1-shot and 5-shot settings for at most 8 rounds.}
	\label{fig:robust}
	\vspace{-10px}
\end{figure}


%% file: section/sec-conclusion.tex
 \vspace{-5px}
\section{Conclusion}\label{sec:conclusion}
In this paper, we propose GORAG, a Graph-based Online Retrieval Augmented Generation framework for the Dynamic Few-shot Social media Text Classification (DFSTC) task. 
Extensive experiments on hate speech and misinformation detection benchmarks with different characteristics demonstrate the effectiveness of GORAG 
in classifying social media texts with only a limited number or even no labeled data. 
GORAG shows its effectiveness in adapting to the dynamic updates of target labels by adaptively retrieving candidate labels to filter the large target label set in each update round. 
For future work, we aim to further enhance GORAG's performance and explore its application in more general scenarios.

%% file: section/sec-appendix.tex
\section{Notations}
\label{sec:notations}

Due to the space limitations, we put the summary table on important notations here in \autoref{tab:notation}. 

\begin{table}[h]
	\centering
	\caption{Summary on Important Notations.}
	\label{tab:notation}
	\setlength\tabcolsep{4pt}
	\resizebox{\linewidth}{!}{
		\begin{tabular}{l|p{7cm}}
			\hline\hline
			
			\textbf{Notations}   & \textbf{Meanings}  \\ 
			\hline
			
			$t$ & Text. \\  \hline
			
			$w$ & The edge weight.  \\ 		\hline
			
			$r$ & The round of label updating.   \\ \hline
			
			$\tau$ & The keyword extraction threshold.   \\ \hline
			
			$y \in \mathcal{Y}^r$ & The target label at round $r$.  \\ \hline
			
			$k$  &  The number of labeled data provided per target label. \\ \hline

			$f^r_\theta$ & The function that learn all target labels' score for texts. \\ 		\hline
			
			$\mathbf{s}^r$ &  The score for all target labels at round $r$. \\ \hline
			
			$y^{r*}$ &   	The predicted label for the text at round $r$.\\ 	\hline
			
			$\mathcal{G}^r(\mathcal{V},\mathcal{E})$ & weighted graph with node set  $\mathcal{V}$ and edge set $\mathcal{E}. $ \\ \hline
			
			$\mathcal{Y}^r$ &  New target labels at round $r$. \\ \hline
			
			
			$\mathcal{D}^r$ & The labeled training text for all new labels at round $r$.  \\ \hline
			
			
			$	\mathcal{T}^r$ &	Unlabeled texts at round $r$. \\ 		\hline		
			
			$	\mathcal{T}_{test}^r$ & The query text for at round $r$.	\\ 		\hline	
			
			$	\mathcal{K}^r$ &	The extracted keyword set for new labels of round $r$.	\\ \hline

			$	\mathcal{V}_{new}^r$ & The set of new graph nodes to be added in round $r$.	\\ 		\hline					
			$\mathcal{N}(\cdot)$& The neighbor set of a node in the weighted graph.	\\ 		\hline		
			
			$p_{extract}$&	The extraction instruction prompt.	\\ \hline					
			$p_{gen}$& The generation instruction prompt. 	\\ 	\hline		
			$p_{classify}$&The classification instruction prompt. 	\\ 	\hline		
			
			$CS(\cdot)$ &  	The correlation score between keywords and labels. \\ 		\hline			
			
			$w^r_{v,y}$&  The weight of edge $e_{v,y}$ in round $r$.	\\ 		\hline			
			$P_{u,v}$&  The shortest path between node $u$ and node $v$.	\\ 		\hline			
			$\mathcal{V}^t_{exist}$& The keywords extracted from $t$ and exist in graph.	\\ 		\hline		
			
			$\mathcal{V}^t_{notexist}$ & The keywords extracted from $t$ and not exist in graph.	\\ 		\hline

			$\mathcal{E}_{oi}^t$ &  Edges added based on text $t$. \\ 
			\hline\hline
		\end{tabular}
	}
\end{table}

\section{Proof}
\label{sec:proof}
Here we prove our Adaptive Candidate Label Generation problem can be reduced from the NP-hard Steiner Tree problem \cite{Su2020}.
\begin{theorem}
	The Adaptive Candidate Label Generation problem is NP-hard. 
\end{theorem}
\begin{proof}
		To demonstrate that the Adaptive Candidate Label Generation problem is NP-hard, we provide a simple reduction of our problem from the Steiner Tree problem. Since the Steiner Tree problem is already proven to be NP-hard \cite{Su2020}, we show that there is a solution for the Steiner Tree problem if and only if there is a solution for our problem. Firstly, given a solution $S$ for our problem, we can construct a Steiner Tree by generating a minimum spanning tree for all nodes in $S$ on graph $\mathcal{G}^r$ then connecting all nodes from $\mathcal{V}_{exist}^t$ to their closest neighbor nodes in $S$. Secondly, any Steiner Tree $ST$ that contains any node $y\in\mathcal{V}^r_t$ is also a solution of our problem with $|\{y\in\mathcal{V}^r_t\}|$ nodes.
		Thus, we prove that the adaptive candidate type generation problem is NP-hard. 
\end{proof}

\section{Pseudo Code of GORAG}
\label{sec:pseduo_code}
In this section, we provide the pseudo code of GORAG's graph construction and adaptive candidate type generation algorithm in \autoref{alg:graph_construction} and \autoref{alg:mehlhorn}, respectively. 

\section{Complexity Analysis}
\label{sec:complexity}
Firstly, we analyze the time and space complexity of GORAG's graph construction. For the text keyword extraction procedure, the time complexity would be $O(|\mathcal{D}^r|)$; If label names are available, the time complexity of generating label descriptions and extract label keywords would be $O(|\mathcal{Y}^r|)$; Calculating the TFIDF and indexing edges to graph would require $O(|\mathcal{V}_{new}^r||\mathcal{D}^r|)$ times, and merging graph requires $O(|\mathcal{Y}^r||\mathcal{Y}^{r-1}|)$. Hence, the total time complexity of GORAG's graph construction at the $r$-th round would be: 

\begin{center}
	$O(|\mathcal{Y}^r||\mathcal{Y}^{r-1}|+|\mathcal{V}_{new}^r||\mathcal{D}^r|)$. 
\end{center}

For the space complexity, the graph is stored with the weighted adjacency matrix, hence needs $O(|\mathcal{E}^r|)$ space; Storing the training corpus at the $r$-th round would need $O(|\mathcal{D}^r|)$ space; 
Storing the representative keywords $\mathcal{K}_{\hat{\mathcal{Y}_t^r}}$ would need $O(|\mathcal{K}_{\hat{\mathcal{Y}_t^r}}|)$ space. Hence the total space complexity of GORAG's graph construction at the $r$-th round would be: 

\begin{center}
	$O(|\mathcal{E}^r|+|\mathcal{D}^r|+|\mathcal{K}_{\hat{\mathcal{Y}_t^r}}|)$. 
\end{center}


The time complexity of GORAG's adaptive candidate type generation algorithm is the same with the Mehlhorn algorithm, which is $O(|\mathcal{E}^r|+|\mathcal{V}^r|log|\mathcal{V}^r|)$ \cite{MEHLHORN1988125}; The time complexity of the online indexing mechanism would cost $|\mathcal{V}_{notexist}^r||\mathcal{T}_{test}^r|$; 
Hence, the total time complexity of GORAG's adaptive retrieval can be denoted as
\begin{center} $O(|\mathcal{T}_{test}^r|(|\mathcal{E}^r|+|\mathcal{V}_{notexist}^r||\mathcal{T}_{test}^r|+|\mathcal{V}^r|log|\mathcal{V}^r|)).$
\end{center}

For the space complexity, $O(|\mathcal{E}^r|)$ to store the graph, and $O(|\mathcal{T}_{test}^r|)$ space is needed to store the testing corpus. Hence, the total space complexity of GORAG's Retrieval and Classification procedure is 
\begin{center}
	$O(|\mathcal{E}^r|+|\mathcal{T}_{test}^r|).$
\end{center}

\begin{algorithm}[t]
	\caption{Graph Construction Algorithm of GORAG at round $r$}\label{alg:graph_construction}
	\begin{algorithmic}[1]
		\REQUIRE $\mathcal{D}^r$: Training text set at the $r$-th round and the provided label $y_i \in \mathcal{Y}^r$ for each training text $t_j$; 
		
		$p_{extraction},\ p_{gen}$: The extraction and description generation instruction prompt; 
		
		$\mathcal{G}^{r-1}$: Weighted graph of the previous round $r-1$.
		\ENSURE $\mathcal{G}^r$: Constructed weighted graph at the $r$-th round.
		\STATE Let $\mathcal{V}_{text}^r=\emptyset$, $\mathcal{V}_{label}^r=\emptyset$, $\mathcal{E}_{new}^r=\emptyset$
		\FOR{each text $t\in \mathcal{T}^r$}
		\STATE $\mathcal{V}_{text}^r=\mathcal{V}_{text}^r\cup KeywordExtract(t) \cup y_t$
		\ENDFOR
		\IF{label names are available}
		\FOR{each label $y_i\in \mathcal{Y}^r$}
		\STATE Get some label related keywords $\mathcal{V}_{label}^r=LLM(p_{gen},\mathcal{D}^r)\cup \mathcal{V}_{label}^r$.
		\ENDFOR
		\ENDIF
		\STATE Let $\mathcal{V}_{new}^r=\mathcal{V}_{text}^r \cup \mathcal{Y}^r_{new} \cup \mathcal{V}_{label}^r$ be the new node set at the $r$-th round.
		\FOR{each node $v\in\mathcal{V}_{new}^r$}
		\STATE $\mathcal{E}_{new}^r=\mathcal{E}_{new}^r\cup e_{v,y_i}$
		\STATE $\mathcal{W}_{new}^r=\mathcal{W}_{new}^r\cup w_{v,y_i}^r$
		\ENDFOR
		\STATE Let $\mathcal{G}^r_{new}=(\mathcal{V}_{new}^r,\mathcal{E}_{new}^r)$ be the newly constructed graph for at the $r$-th round.
		\STATE Let $\mathcal{G}^r=Merge(\mathcal{G}^r_{new}, \mathcal{G}^{r-1})$ be the final constructed graph at the $r$-th round.
		\RETURN  $\mathcal{G}^r$
	\end{algorithmic}
\end{algorithm}

\section{Experiment Details}\label{sec:exp_detail}
\subsection{Introduction of chosen baselines}
\label{sec:baselines}
	As shown in \autoref{tab:main_exp},  we compare GORAG's performance with 7 baselines from 3 technical categories for few-shot experiments, they are NN-based Entailment \cite{xia-etal-2021-incremental}, Long Context-based NaiveRAG \cite{gao2023retrieval}, LongLLMLingua \cite{jiang-etal-2024-longllmlingua}, and Graph-based  GraphRAG \cite{edge2024local}, LightRAG \cite{guo2024lightrag}, and HippoRAG2 \cite{gutiérrez2025ragmemory}.

	\noindent \textbf{NN-based Models}
	\begin{itemize}[leftmargin=10pt]
		\item \textbf{Entailment} \cite{xia-etal-2021-incremental}: Entailment concatenates the text data with each of the label names to form multiple entailment pairs with one text sample, hence increasing the number of training data and enhance its finetuning of a RoBERTa PLM \cite{liu2019roberta}. 
		
	\end{itemize} 
	\noindent \textbf{Long Context RAG Models}
	\begin{itemize}[leftmargin=10pt]
		\item \textbf{NaiveRAG} \cite{gao2023retrieval}: It acts as a foundational baseline of current RAG models. When indexing, it stores text segments in a vector database using text embeddings. When querying, retrieves the side information based on their vector similarity to the query.
		
		\item \textbf{LongLLMLingua} \cite{jiang-etal-2024-longllmlingua}: LongLLMLingua is a instruction aware prompt compressor model, it applies LLM's generation perplexity to filter out un-important tokens of the model input based on the retrieved side information and the task instruction. 
		
	\end{itemize} 
	\noindent \textbf{Graph-based RAG Models}
	\begin{itemize}[leftmargin=10pt]
		\item \textbf{GraphRAG} \cite{edge2024local}: It employs LLM for graph construction. When retrieval, it aggregates nodes into communities w.r.t. the query, and generates community reports to encapsulate global information from texts.
		
		\item \textbf{LightRAG} \cite{guo2024lightrag}: 
		It calculates the vector similarity between query extracted entities and graph nodes, achieving a one-to-one mapping from entities to graph nodes, then retrieve these nodes' graph neighborhood.
		
		
		
		\item \textbf{HippoRAG2} \cite{gutiérrez2025ragmemory}: 
		It employs LLM for graph construction, and also add the text chunks as nodes in the graph.  
		As a result, it directly retrieves Top-k text chunks based on their respective nodes' PPR score.
	\end{itemize} 

	\begin{algorithm}[t!]
	\caption{Adaptive Candidate Type Generation Algorithm}
	\label{alg:mehlhorn}
	\begin{algorithmic}[1]
		\REQUIRE $\mathcal{G}^r (\mathcal{V}^r, \mathcal{E}^r,\mathcal{W}^r)$: The constructed weighted graph;
		
		$\mathcal{V}_{exist}^t \subseteq \mathcal{V}^r$: A set of keyword nodes extracted from text $t$ and can be mapped to graph $\mathcal{G}$; 
		
		$\mathcal{Y}^r$: The target label set at the $r$-th round.
		\ENSURE $\hat{\mathcal{Y}}^r_t$: The candidate type retrieved for text $t$.
		\STATE Compute a minimum spanning tree $MST$ of the graph $\mathcal{G}$.
		\STATE Let $\mathcal{V}_{term} = \mathcal{V}_{exist}^t \cap  \mathcal{V}(MST)$ be the terminal nodes in $MST$
		\STATE Construct a weighted auxiliary graph $H$: $ \mathcal{V}(H) = \mathcal{V}_{term}$
		\FOR{each pair of terminals $u, v \in \mathcal{V}_{term}$}
		\STATE Find the shortest path $P_{uv}$ in $MST$ from $u$ to $v$
		\STATE Let $w^H_{u, v} = \sum_{e \in P_{uv}} w_e$
		\ENDFOR
		\STATE Compute a minimum spanning tree $MST'$ of the auxiliary graph $H$
		\STATE Let $ST = \emptyset$
		\FOR{each edge $e_{u, v} \in MST'$}
		\STATE Add the shortest path $P_{uv}$ in $MST$ to $ST$
		\ENDFOR
		\STATE Let $\hat{\mathcal{Y}}^r_t =ST \cap \mathcal{Y}^r$
		\RETURN $\hat{\mathcal{Y}}^r_t$
	\end{algorithmic}
\end{algorithm}
\subsection{Hardware Settings}
All experiments are conducted on an Intel(R) Xeon(R) Gold 5220R @ 2.20GHz CPU and a single NVIDIA A100-SXM4-40GB GPU.
\subsection{Case Study}
\label{ssec:exp:case}

In this section, we dig into a test case from the CAD social media abuse detection dataset to demonstrate the ability of GORAG to provide graph-grounded clues, to explain its classification result.

As shown in \autoref{fig:case_study}, the keyword entity \textcolor{red}{feminists} and \textcolor{red}{sexual predators} are extracted from the input unlabeled text and mapped to the graph. 
GORAG then generates a Minimum Cost Steiner Tree through these mapped entities, which also spans through the candidate label \textbf{Insult}. 
The structure and involved nodes of this generated Steiner Tree (edges in \textcolor{red}{red}) serve as an explicit clue, helping the LLM to classify the text correctly as Insulting Content.

\begin{figure}[t]
	\centering
	\includegraphics[width=0.9\linewidth]{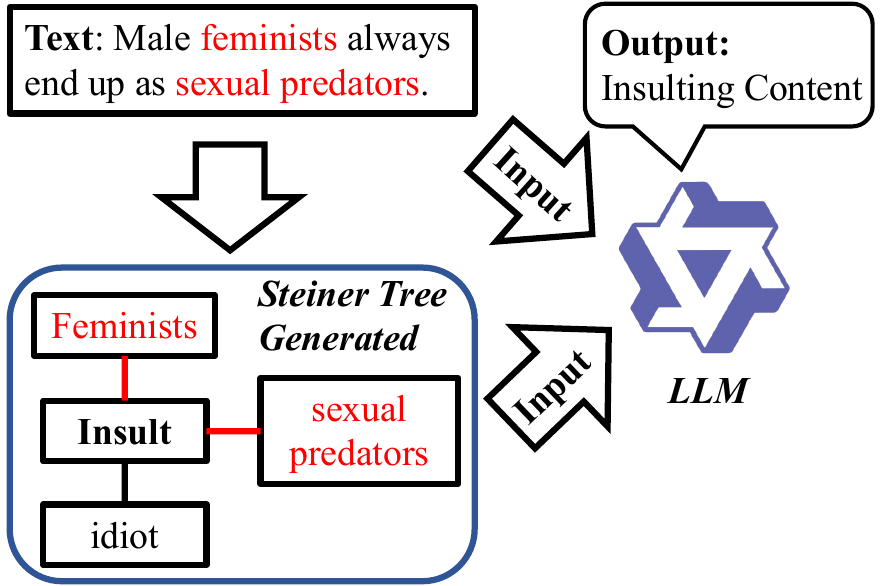}
	\caption{A case from the CAD dataset, where the generated Steiner Tree serves as an explicit clue, helping the LLM to classify the text correctly.}
	\label{fig:case_study}
	\vspace{-10px}
\end{figure}